\documentclass[10pt,twocolumn,letterpaper]{article}

\usepackage{iccv}
\usepackage{times}
\usepackage{epsfig}
\usepackage{graphicx}
\usepackage{amsmath}
\usepackage{amssymb}
\usepackage{mathtools}
\usepackage{amsthm}
\usepackage{xcolor}         
\usepackage{algorithm, algorithmic}
\usepackage{bm}
\usepackage{amsmath}
\usepackage{bbm}
\usepackage{adjustbox}
\usepackage{float}
\usepackage{multirow}

\usepackage{graphicx}
\usepackage{subcaption}
\usepackage{thmtools, thm-restate}
\usepackage{comment}

\usepackage[pagebackref=true,breaklinks=true,letterpaper=true,colorlinks,bookmarks=false]{hyperref}

\iccvfinalcopy 

\def\iccvPaperID{11705} 

\ificcvfinal\fi

\declaretheorem[name=Lemma, numberwithin=section]{lemma}

\begin{document}

\title{Enhancing Adversarial Robustness in Low-Label Regime \\ via Adaptively Weighted Regularization and Knowledge Distillation}

\author{Dongyoon Yang\\
Seoul National University\\
Department of Statistics\\
{\tt\small ydy0415@gmail.com}
\and
Insung Kong\\
Seoul National University\\
Department of Statistics\\
{\tt\small ggong369@snu.ac.kr}
\and
Yongdai Kim\\
Seoul National University\\
Department of Statistics\\
{\tt\small ydkim0903@gmail.com}
}

\maketitle

\begin{abstract}
Adversarial robustness is a research area that has recently received a lot of attention in the quest for trustworthy artificial intelligence.
However, recent works on adversarial robustness have focused on supervised learning
where it is assumed that  labeled data is plentiful.
In this paper, we investigate semi-supervised adversarial training where labeled data is scarce.
We derive two upper bounds for the robust risk and propose a regularization term for unlabeled data motivated by these two upper bounds.
Then, we develop a semi-supervised adversarial training algorithm that combines the proposed regularization term with knowledge distillation using a semi-supervised teacher (i.e., a teacher model trained using a semi-supervised learning algorithm).
Our experiments show that our proposed algorithm achieves state-of-the-art performance with significant margins compared to existing algorithms.
In particular, compared to supervised learning algorithms, performance of our proposed algorithm is not much worse even when the amount of labeled data is very small.
For example, our algorithm with only 8\% labeled data is comparable to supervised adversarial training algorithms that use all labeled data, both in terms of standard and robust accuracies on CIFAR-10.
\end{abstract}


\section{Introduction}
\label{intro}

Neural network models used for image classification are vulnerable to adversarial perturbations that are imperceptible to humans \cite{szegedy2014intriguing}. These perturbed images are called \textit{adversarial examples}, and they can be generated without any knowledge of the underlying model, leading to security concerns \cite{papernot2016transferability, papernot2017practical, chen2017zoo, ilyas2018blackbox, papernot2016science}. Adversarial examples can also cause problems in real-world scenarios, where printed images with adversarial perturbations can easily fool the classification model \cite{kurakin2016adversarial}.
To defend against adversarial attacks, many adversarial learning algorithms have been proposed, such as \cite{madry2018towards, zhang2019theoretically, wang2020improving, zhang2020attacks, zhang2021geometry, rade2022reducing}. 


Learning accurate prediction models typically requires a large amount of labeled data, which can be expensive and time-consuming to collect. In contrast, obtaining unlabeled data is relatively easier. Semi-supervised learning is a research area that focuses on effectively utilizing unlabeled data \cite{miyato2017virtual, berthelot2019mixmatch, unsupervised2020xie, sohn2020fixmatch, zhang2021flexmatch}.
Because adversarial training algorithms also require labeled data,
semi-supervised adversarial training (SS-AT) has become a crucial research area for reliable artificial intelligence \cite{carmon2019unlabeled, uesato2019are, zhai2019adversarially, zhang2021armoured}. 

The aim of this paper is to develop a new SS-AT algorithm that is theoretically well motivated and empirically superior to existing competitors when the amount of labeled data is insufficient. We propose an objective function for SS-AT, which consists of three key components - one for generalization of labeled data, a regularization term with unlabeled data for adversarial robustness, and a knowledge distillation term for improving generalization ability of unlabeled data.
Our regularization term is motivated by two new upper bounds of the boundary risk.
The knowledge distillation term is added to estimate soft pseudo labels of unlabeled data, which
contrasts with existing algorithms \cite{uesato2019are, carmon2019unlabeled, zhai2019adversarially} that use hard pseudo labels. 
Using soft pseudo labels instead of hard pseudo labels improves the performance significantly.

Our contributions can be summarized as follows:
\begin{itemize}
    \item We derive two upper bounds of the robust risk for semi-supervised adversarial training, providing theoretical insights into the performance of proposed method.
    \item We propose a novel semi-supervised adversarial training algorithm, called Semisupervised-Robust-Self-Training with Adaptively Weighted Regularization (SRST-AWR), that combining an adaptively weighted regularization and 
    knowledge distillation with a semi-supervised teacher to put soft pseudo labels in adversarial training.
    \item We demonstrate the effectiveness of our algorithm through numerical experiments on various benchmark datasets, showing simultaneous improvements in robustness and generalization with significant performance gains over existing state-of-the-art methods.
    \item Our proposed algorithm exhibits only minor performance degradation even when the amount of labeled data is limited compared to fully supervised methods with a large amount of labeled data.
\end{itemize}

\section{Preliminaries}
\label{sec2}

Let $\mathcal{X} \subset \mathbb{R}^d$ be the input space, $\mathcal{Y} = \left\{1, \cdots, C\right\}$
be the set of output labels and $f_{\bm{\theta}} : \mathcal{X} \rightarrow \mathbb{R}^{C}$ be the score function parametrized by parameters $\bm{\theta}$ such that $\mathbf{p}_{\theta}(\cdot|\bm{x}) =\operatorname{softmax}(f_{\bm{\theta}}(\bm{x})) \in \mathbb{R}^C$ is the vector of the predictive 
conditional probabilities. Let $F_{\bm{\theta}}(\bm{x}) = \underset{c}{\mathrm{argmax}} [f_{\bm{\theta}}(\bm{x})]_c \in \mathbb{R}^C ,$  
$\mathcal{B}_{p}(\bm{x}, \varepsilon) = \left\{\bm{x}' \in \mathcal{X} : \lVert \bm{x}- \bm{x}' \rVert_p \leq \varepsilon \right\}$ and Let $\mathbbm{1}\{\cdot\}$ represent the indicator function, which takes the value $1$ when the condition $\cdot$ is satisfied, and $0$ otherwise.

\subsection{Population Robust Risk}

The population robust risk used in adversarial training is defined as
\begin{equation}
\label{rob-pop}
\mathcal{R}_{\text{rob}}(\theta)=\mathbb{E}_{\mathbf{(X,Y)}} \;  \underset{\mathbf{X'} \in \mathcal{B}_p(\mathbf{X}, \varepsilon) }{\max\;\;\;} \mathbbm{1}\left\{F_{\theta}(\mathbf{X'}) \neq \mathbf{Y} \right\}.
\end{equation}
The objective of adversarial training is to learn $\bm{\theta}$ that minimizes the population robust risk (\ref{rob-pop}). Most state-of-the-art adversarial training algorithms consist of two steps: a maximization step and a minimization step. In the maximization step, an adversarial example $\bm{x}' \in \mathcal{B}_p(\bm{x}, \varepsilon)$ is generated, which is described in Section \ref{adversarial-attack}.
In the minimization step, 
a certain regularized empirical risk for given adversarial examples is minimized \cite{madry2018towards, zhang2019theoretically, wang2020improving, rade2022reducing}.

\subsection{Adversarial Attack}
\label{adversarial-attack}
An adversarial attack is a method to generate an adversarial example. Adversarial attacks can be categorized into white-box attacks \cite{szegedy2014intriguing, goodfellow2015explaining, madry2018towards} and black-box attacks \cite{papernot2016transferability, papernot2017practical, andriushchenko2020square}. In the white-box attack, it is assumed that the adversary can exploit all information about the model architectures and parameters to generate adversarial examples. In a black-box attack, the adversary can only access the  outputs of the model. 

One of the most popular white-box adversarial attack algorithms is Projected Gradient Descent (PGD), which finds an adversarial example 
by iteratively updating it by the gradient ascent
and projecting onto the $\epsilon$-ball of the original data \cite{madry2018towards}.
The formula of PGD$^{\text{T}}$ is as follows:
\begin{equation}
\footnotesize
\label{pgd}
    \bm{x}^{(T)}=\bm{\Pi}_{\mathcal{B}_{p}(\bm{x}, \varepsilon) }\left(\bm{x}^{(T-1)} + \nu \operatorname{sgn}\left(\nabla_{\bm{x}^{(T-1)}} \eta(\bm{x}^{(T-1)}|\bm{\theta},\bm{x},y)\right)\right),
\end{equation}
where $\bm{\Pi}_{\mathcal{B}_{p}(\bm{x}, \varepsilon)}(\cdot)$ is the projection operator to $\mathcal{B}_{p}(\bm{x}, \varepsilon)$, $\nu>0$ is the step size, $\eta$ is a surrogate loss
, $\bm{x}^{(0)}=\bm{x}$,
$\text{T}$ is the number of iterations.
The cross-entropy or Kullback–Leibler divergence for $\eta$ can be used.

\subsection{Semi-Supervised Learning}

Virtual Adversarial Training (VAT) \cite{miyato2017virtual} is a semi-supervised learning algorithm minimizing
\begin{equation}
\small
\label{vat}
    \frac{1}{n_l} \sum\limits_{i=1}^{n_l} \ell_{\text{ce}}(f_{\theta}(\bm{x}_i), y_i) + \lambda \frac{1}{n_{ul}} \sum\limits_{j=1}^{n_{ul}} \operatorname{D_{KL}} (\mathbf{p}_{\tilde{\theta}}(\cdot|\bm{x}_j)\lVert \mathbf{p}_{\theta }(\cdot|\widehat{\bm{x}}^{\text{adv}}_j)),
\end{equation}
where $\{(\bm{x}_i, y_i)\}^{n_l}_{i=1}$ and $\{\bm{x}_j\}_{j=1}^{n_{ul}} $ are labeled and unlabeled samples, respectively, and
$\tilde{\theta}$ is the pretrained parameter and $\widehat{\bm{x}}^{\text{adv}}_j \in \mathcal{B}_2(\bm{x}_j, \varepsilon)$ is an adversarial example.

FixMatch \cite{sohn2020fixmatch} is a semi-supervised learning algorithm minimizing
\begin{align*}
\small
\label{fixmatch}
    & \frac{1}{n_l} \sum\limits_{i=1}^{n_l} \ell_{\text{ce}}(f_{\theta}(\bm{x}_i), y_i) \\ 
    & + \frac{1}{n_{\tau}} \sum\limits_{j=1}^{n_{ul}} \ell_{\text{ce}} (f_{\bm{\theta}}(\bm{x}^{s}_j), F_{\bm{\theta}}(\bm{x}^{w}_j)) \mathbbm{1}\{\underset{c}{\max}\; p_{\bm{\theta}}(c| \bm{x}^{w}_j) > \tau \},
\end{align*}
where $\tau \in (0,1) $ is a constant, 
$ n_{\tau} = \sum\limits_{j=1}^{n_{ul}} \mathbbm{1}\{\underset{c}{\max}\; p_{\bm{\theta}}(c| \bm{x}^{w}_j)$ $ > \tau \}$ and
$\bm{x}^s_j$ and $\bm{x}^{w}_j$ are strongly and weakly augmented samples \cite{cubuk2018autoaugment}, respectively.

\begin{figure*}[t]
\begin{center}
\begin{minipage}[c]{0.45\linewidth}
\includegraphics[width=\linewidth]{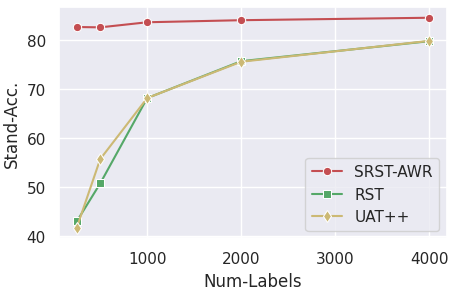}
\end{minipage}
\hfill
\begin{minipage}[c]{0.45\linewidth}
\includegraphics[width=\linewidth]{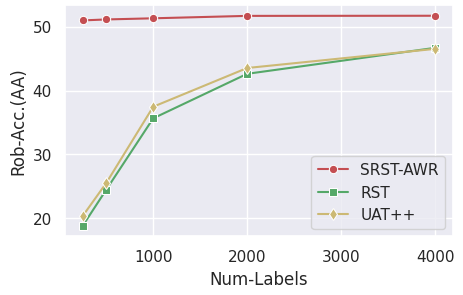}
\end{minipage}
\hfill
\caption{\textbf{Performance comparison of SRST-AWR, RST and UAT++ for varying the number of labeled data}.
The $x$-axis is the number of labeled data and $y$-axis are the standard accuracy and robust accuracy against autoattack, respectively.
}
\label{fig:numlabels}
\end{center}
\end{figure*}
\subsection{Semi-Supervised Adversarial Training}
\label{ssar}

Existing adversarial training algorithms can be categorized into two types: one directly minimizing empirical robust risk (e.g. PGD-AT \cite{madry2018towards}), and the other decomposing the robust risk into supervised and regularization terms and minimizing the corresponding regularized empirical risk (e.g. TRADES \cite{zhang2019theoretically} ). In algorithms based on PGD-AT, label information for all data is required, and thus it cannot be directly applied in a semi-supervised setting. TRADES can be applied in a semi-supervised setting since the regularization term does not require label information. However, it shows poor performance for semi-supervised learning. 

For achieving adversarial robustness in a semi-supervised setting, several SS-AT algorithms have been proposed \cite{uesato2019are, carmon2019unlabeled, zhai2019adversarially, zhang2021armoured, gowal2021selfsupervised}.
RST \cite {carmon2019unlabeled} generates pseudo labels for unlabeled data by predicting the class labels using a teacher model trained only with labeled data. Then, the algorithm minimizes the regularized empirical risk of TRADES \cite{zhang2019theoretically} using both the labeled data and the unlabeled data with their pseudo labels.
That is, it minimizes the following regularized empirical risk:
\begin{align}
    &\mathcal{R}_{\text{(RST)}}(\bm{\theta} ; \left\{(\bm{x}_i, y_i) \right\}_{i=1}^{n_l}, \left\{(\bm{x}_j, \widehat{y}_j) \right\}_{j=n_{l}+1}^{n_{n_l} + n_{ul}}, \lambda ) \nonumber \\
    & := \frac{1}{n_{l}+n_{ul}}\sum\limits_{k=1}^{n_l + n_{ul}} \Big\{ \ell_{\text{ce}}(f_{\theta}(\bm{x}_k), y_k \text{\;or\;} \widehat{y}_k) \nonumber \\
    & + \lambda \operatorname{D_{KL}}(\mathbf{p}_{\bm{\theta}}(\cdot|\bm{x}_k) || \mathbf{p}_{\bm{\theta}}(\cdot|\widehat{\bm{x}}^{\text{pgd}}_k)) \Big\} ,
\end{align}
where $\widehat{y}_j = F_{\bm{\theta}_{T}}(\bm{x}_j)$ and $\bm{\theta}_{T}$ is the parameter of the teacher model trained only with labeled data.

On the other hand, UAT \cite{uesato2019are} minimizes the following regularized empirical risk:
\begin{align}
    &\mathcal{R}_{\text{(UAT)}}(\bm{\theta} ; \left\{(\bm{x}_i, y_i) \right\}_{i=1}^{n_l}, \left\{(\bm{x}_j, \widehat{y}_j) \right\}_{j=n_{l}+1}^{n_{n_l} + n_{ul}}, \lambda ) \nonumber \\
    & := \frac{1}{n_{l}+n_{ul}}\sum\limits_{k=1}^{n_l + n_{ul}} \Big\{ \ell_{\text{ce}}(f_{\theta}(\widehat{\bm{x}}^{\text{pgd}}_k), y_k \text{\;or\;} \widehat{y}_k) \nonumber \\
    & + \lambda \operatorname{D_{KL}}(\mathbf{p}_{\bm{\tilde{\theta}}}(\cdot|\bm{x}_k) || \mathbf{p}_{\bm{\theta}}(\cdot|\widehat{\bm{x}}^{\text{pgd}}_k)) \Big\}.
\end{align}
If $\lambda=0$, it is called Unsupervised Adversarial Training with Fixed Targets (UAT-FT); otherwise, it is called Unsupervised Adversarial Training Plus Plus (UAT++).

However, as seen in Figure \ref{fig:numlabels}, RST and UAT++ show poor performances when the amount of labeled data is insufficient, which is partially the teacher model does not predict well.
To resolve this problem,
ARMOURED (Adversarially Robust MOdels using Unlabeled data by REgularizing Diversity) \cite{zhang2021armoured}
combines multi-view learning and diversity regularization. 
It employs a multi-view ensemble learning approach for selecting high quality pseudo labeled data.

\section{Semi-Supervised Robust Self-Training via Adaptively Weighted Regularization and Knowledge Distillation}

In this section, we derive two upper bounds of the robust risk and propose a SS-AT algorithm by modifying the upper bounds.

\subsection{Upper Bounds of the Population Robust Risk}

The population robust risk $\mathcal{R}_{\text{rob}}(\bm{\theta})$ is decomposed of the two terms - natural risk and boundary risk as follows \cite{zhang2019theoretically}:
\begin{equation*}
    \mathcal{R}_{\text{rob}}(\bm{\theta}) = \mathcal{R}_{\text{nat}}(\bm{\theta}) + \mathcal{R}_{\text{bdy}}(\bm{\theta}),
\end{equation*}
where $\mathcal{R}_{\text{nat}}(\bm{\theta}) = \mathbb{E}_{(\mathbf{X}, Y)}\mathbbm{1}\left\{ F_{\bm{\theta}}(\mathbf{X}) \neq Y \right\}$ and $\mathcal{R}_{\text{bdy}}(\bm{\theta}) = \mathbb{E}_{(\mathbf{X}, Y)}\mathbbm{1}\{\exists \mathbf{X}' \in \mathcal{B}_p(\mathbf{X}, \varepsilon) :F_{\bm{\theta}}(\mathbf{X})\neq F_{\bm{\theta}}(\mathbf{X}')   ,F_{\bm{\theta}}(\mathbf{X}) = Y\}$.
The objective of adversarial training is to find $\bm{\theta}$ minimizing the population robust risk.

The following theorems provide two upper bounds of the population robust risk whose proofs are deferred to Appendix \ref{appA}.

\begin{restatable}{theorem}{semiarow}
\label{thm3.1}
For a given score function $f_{\bm{\theta}},$ 
let
\begin{equation*}
    z(\bm{x}) \in \underset{\bm{x}' \in \mathcal{B}_{p}(\bm{x}, \varepsilon)}{\operatorname{argmax}} \mathbbm{1} \left\{ F_{\bm{\theta}}(\bm{x}) \neq F_{\bm{\theta}}(\bm{x}')\right\}.
\end{equation*}
Then, we have
\begin{align} 
    \mathcal{R}_{\text{rob}}&(\bm{\theta})  \leq \mathbb{E}_{(\mathbf{X},Y)} \mathbbm{1}\{ Y \neq F_{\bm{\theta}}(\mathbf{X}) \} \nonumber \\
     + &\mathbb{E}_{\mathbf{X}} \left\{ {\mathbbm{1}\{ F_{\bm{\theta}}(\mathbf{X}) \neq F_{\bm{\theta}}(z(\mathbf{X})) \} } \cdot p(Y \neq F_{\bm{\theta}}(z(\mathbf{X})) | \mathbf{X}) \right\}.
    \label{semi-arow:ub}
\end{align}
\end{restatable}

\begin{restatable}{theorem}{semicow}
\label{thm3.2}
For a given score function $f_{\bm{\theta}},$ 
let
\begin{equation*}
    z(\bm{x}) \in \underset{\bm{x}' \in \mathcal{B}_{p}(\bm{x}, \varepsilon)}{\operatorname{argmax}} \mathbbm{1} \{ F_{\bm{\theta}}(\bm{x}) \neq F_{\bm{\theta}}(\bm{x}')\}.
\end{equation*}
Then, we have
\begin{align}
    \mathcal{R}_{\text{rob}}&(\bm{\theta})  \leq \mathbb{E}_{(\mathbf{X},Y)} \mathbbm{1}\{ Y \neq F_{\bm{\theta}}(\mathbf{X}) \} \nonumber \\
    + & \mathbb{E}_{\mathbf{X}} \left\{ {\mathbbm{1} \{F_{\bm{\theta}}(\mathbf{X}) \neq F_{\bm{\theta}}(z(\mathbf{X}))\}} \cdot p(Y = F_{\bm{\theta}}(\mathbf{X}) | \mathbf{X}) \right\}.
    \label{semi-cow:ub}
\end{align}
\end{restatable}

The key point of Theorems \ref{thm3.1} and \ref{thm3.2} is that their second terms on the right-hand side do not depend on label information. 
Thus, the second terms can be used as regularization term for unlabeled data.
In contrast, the regularization terms used in
supervised adversarial training algorithms
such as MART \cite{wang2020improving} and ARoW \cite{dongyoon2023improving} require label information.





\paragraph{Comparison to TRADES \cite{zhang2019theoretically}}
For binary classification problems such that $\mathcal{Y}=\{-1, 1\}$, the following upper bounds can be derived using Theorems \ref{thm3.1} and \ref{thm3.2}: 
\begin{align}
    \mathcal{R}_{\text{rob}}&(\bm{\theta}) 
     \leq \mathbb{E}_{(\mathbf{X},Y)} \phi(Y f_{\bm{\theta}}(\mathbf{X})) \nonumber \\
    + & \mathbb{E}_{\mathbf{X}} \{ {\phi(f_{\bm{\theta}}(\mathbf{X}) f_{\bm{\theta}}(z(\mathbf{X})) / \lambda )} \cdot p(Y \neq F_{\bm{\theta}}(z(\mathbf{X})) | \mathbf{X})\} \label{semi-arow:binary}, \\
    \mathcal{R}_{\text{rob}}&(\bm{\theta}) 
    \leq \mathbb{E}_{(\mathbf{X},Y)} \phi(Y f_{\bm{\theta}}(\mathbf{X})) \nonumber \\
    + & \mathbb{E}_{\mathbf{X}}\{{\phi(f_{\bm{\theta}}(\mathbf{X}) f_{\bm{\theta}}(z(\mathbf{X})) / \lambda)} \cdot
    p(Y = F_{\bm{\theta}}(\mathbf{X}) | \mathbf{X})\} \label{semi-cow:binary},
\end{align}
where $\phi$ is the binary cross entropy loss and $\lambda > 0$ is a regularization parameter. 
The proofs are provided in Appendix \ref{appA}.
In contrast, \cite{zhang2019theoretically} shows that
\begin{equation} 
    \mathcal{R}_{\text{rob}}(\bm{\theta}) \leq \mathbb{E}_{(\mathbf{X},Y)} \phi(Y f_{\bm{\theta}}(\mathbf{X})) +
    \mathbb{E}_{\mathbf{X}}{\phi(f_{\bm{\theta}}(\mathbf{X}) f_{\bm{\theta}}(z(\mathbf{X})) / \lambda) }. \label{trades:binary}
\end{equation}
Note that the upper bounds (\ref{semi-arow:binary}) and (\ref{semi-cow:binary}) are tighter than TRADES \cite{zhang2019theoretically}.



\subsection{Algorithm}


By  modifying the upper bounds in Theorems \ref{thm3.1} and \ref{thm3.2},
we propose the corresponding SS-AT algorithm.
The modifications are as follows:
\begin{itemize}
    \item the adversarial example $z(\bm{x})$ is replaced by $\widehat{\bm{x}}^{\text{pgd}}$ obtained by the PGD algorithm; 
    \item the term $\mathbbm{1}(Y \neq F_{\bm{\theta}}(\mathbf{X}))$
    is replaced by 
    the smooth cross-entropy $\ell^{\text{LS}}(f_{\bm{\theta}}(\bm{x}), y) $, where
    $\ell^{\text{LS}}(f_{\bm{\theta}}(\bm{x}), y) = - {\bm{y}_{\alpha}^{\text{LS}}}^{\top} \log \mathbf{p}_{\theta}(\cdot|\bm{x})$, $\bm{y}_{\alpha}^{\text{LS}} = (1-\alpha)\mathbf{u}_y + \frac{\alpha}{C}\mathbf{1}_C$, $\mathbf{u}_y \in \mathbb{R}^{C}$ is the one-hot vector whose the $y$-th entry is 1 and $\mathbf{1}_C \in \mathbb{R}^{C}$ is the vector whose entries are all 1;
    \item the term $\mathbbm{1}( F_{\bm{\theta}}(\mathbf{X}) \neq F_{\bm{\theta}}(z(\mathbf{X})))$ is replaced by $ \lambda \cdot \operatorname{KL}(\mathbf{p}_{\bm{\theta}}(\cdot|\mathbf{x}) || \mathbf{p}_{\bm{\theta}}(\cdot|\widehat{\mathbf{x}}^{\text{pgd}}))$ for a regularization parameter $\lambda>0;$
    \item the terms $p(Y \neq F_{\bm{\theta}}(\widehat{\mathbf{X}}^{\text{pgd}})|\mathbf{X})$ and $p(Y=F_{\bm{\theta}}(\mathbf{X})|\mathbf{X})$  are replaced by $ 1 - \sum\limits_{c=1}^C p_{\bm{\theta}_T}(Y = c | \widehat{\mathbf{x}}^{\text{pgd}}) p_{\bm{\theta}}( c | \widehat{\mathbf{x}}^{\text{pgd}} )
    $ and $\sum\limits_{c=1}^C p_{\bm{\theta}_T}(Y = c | \mathbf{x}) p_{\bm{\theta}}( c | \mathbf{x} )
    $, respectively.
\end{itemize}



\begin{table*}
    \caption{\textbf{Comparison SRST-AWR, RST and UAT++}. $\bm{\theta}_{\text{sup}}$ and $\bm{\theta}_{\text{semi}}$ are parameters of the models using supervised and semi-supervised learning algorithms, respectively. $\bm{x}_i$ and $\bm{x}_j$ are labeled sample and unlabeled sample, respetively. $x_k$ can be labeled or unlabeled sample having pseudo label $\widehat{y}_k$.}
    \centering
    \setlength{\tabcolsep}{4pt}
    \begin{tabular}{c|ccc}
    \hline
    \textbf{Method}  & Pseudo-labeling & Supervised Term  & Regularization Term  \\
    \hline
    UAT++                    & $\widehat{y}_j = F_{{\bm{\theta}_{\text{sup}}}}(\bm{x}_j)$ &  $\ell_{\text{ce}}(f_{\theta}(\widehat{\bm{x}}^{\text{pgd}}_k) , y_k \text{\;or\;} \widehat{y}_k)$ & $\operatorname{D_{KL}}(\mathbf{p}_{\bm{\tilde{\theta}}}(\cdot|\bm{x}_k) || \mathbf{p}_{\bm{\theta}}(\cdot|\widehat{\bm{x}}^{\text{pgd}}_k))$
    
    \\
    RST                      & $\widehat{y}_j = F_{{\bm{\theta}_{\text{sup}}}}(\bm{x}_j)$ & $\ell_{\text{ce}}(f_{\theta}(\bm{x}_k), y_k \text{\;or\;} \widehat{y}_k)$
    & $\operatorname{D_{KL}}(\mathbf{p}_{\bm{\theta}}(\cdot|\bm{x}_k) || \mathbf{p}_{\bm{\theta}}(\cdot|\widehat{\bm{x}}^{\text{pgd}}_k))$ 
    \\
    ARMOURED & Multi-view learning & $\ell_{\text{ce}}(f_{\theta}(\widehat{\bm{x}}^{\text{pgd}}_k) , y_k \text{\;or\;} \widehat{y}_k)$ & $\mathcal{L}_{\text{DPP}}(\bm{x}_k, y_k \text{\;or\;} \widehat{y}_k)$, $\mathcal{L}_{\text{NEM}}(\bm{x}_k, y_k \text{\;or\;} \widehat{y}_k)$
    \\
    SRST-AWR                & Soft pseudo-labeling & $\ell_{\text{ce}}(f_{\theta}(\bm{x}_i), y_i)$ & $\operatorname{D_{KL}}(\mathbf{p}_{\bm{\theta}}(\cdot|\bm{x}_j) || \mathbf{p}_{\bm{\theta}}(\cdot|\widehat{\bm{x}}^{\text{pgd}}_j)) \times w_{\bm{\theta}}(\bm{x}_j) $  \\
    \hline
  \end{tabular}
  \label{table:compare-surrogate}
  \vskip -0.1in
\end{table*}
 
The upper bounds in Theorems \ref{thm3.1} and \ref{thm3.2} have the conditional probabilities $p(F_{\bm{\theta}}(\widehat{\bm{x}}^{\text{pgd}})|\bm{x})$ and $p(F_{\bm{\theta}}(\bm{x})|\bm{x})$. 
We estimate these conditional probabilities by smooth proxies as following:
\begin{align*}
    p(Y \neq F_{\bm{\theta}}(\widehat{\bm{x}}^{\text{pgd}})| \bm{x}) &= 1 - p(Y = F_{\bm{\theta}}(\widehat{\bm{x}}^{\text{pgd}})| \bm{x}) \\
    &= 1-  \sum\limits_{c=1}^C p(Y=c|\bm{x})\mathbbm{1}(c = F_{\bm{\theta}}(\widehat{\bm{x}}^{\text{pgd}})) \\
    & \approx 1 - \sum\limits_{c=1}^C p_{\bm{\theta}_T}(Y=c|\bm{x})p_{\bm{\theta}}(c | \widehat{\bm{x}}^{\text{pgd}}), \\
    p(Y=F_{\bm{\theta}}(\bm{x})|\bm{x}) &= \sum\limits_{c=1}^C p(Y=c|\bm{x})\mathbbm{1}(c=F_{\bm{\theta}}(\bm{x})) \\
    & \approx \sum\limits_{c=1}^C p_{\bm{\theta}_T}(Y=c|\bm{x})p_{\bm{\theta}}(c | \bm{x}).
\end{align*}
where $\bm{\theta}_T$ is a parameter of the teacher model.

We use the label smooth cross entropy instead of the standard cross entropy during training phase because
the use of standard cross-entropy induces
overconfident predictions \cite{guo2017on}.
To accurately estimate the conditional predictive probability $\mathbf{p}_{\bm{\theta}}(\cdot|\bm{x})$ in modified version of upper bounds in Theorems \ref{thm3.1} and \ref{thm3.2}, we use label smoothing cross-entropy as a surrogate for $\mathbbm{1}(Y \neq F_{\bm{\theta}}(\mathbf{X}))$ \cite{muller2019when}.

Furthermore, we introduce a knowledge distillation term using the teacher model trained by a semi-supervised learning algorithm to facilitate soft pseudo-labeling instead of hard labeling which is used exiting works \cite{carmon2019unlabeled, uesato2019are}.

\begin{algorithm}[H]
    \small
	\textbf{Inputs} : network $f_{\bm{\theta}}$, training dataset $\mathcal{D}^{n_l}_{l}=\{(\bm{x}_i, y_i) \in \mathbb{R}^{d+1}$ $: i=1, \cdots, n_{l} \}$, $\mathcal{D}^{n_{ul}}_{ul}=\{\bm{x}_j \in \mathbb{R}^{d+1}: j=1, \cdots, n_{ul} \}$, learning rate $\eta$, hyperparameters ($\alpha$, $\lambda$, $\beta$, $\gamma$, $\tau$) of (\ref{our:emp-risk}), number of epochs $T$, number of batch $B$, batch size $K$ \\
	\textbf{Output} : adversarially robust model $f_{\bm{\theta}}$
	\begin{algorithmic}[1]
        \STATE Train a teacher model $f_{\bm{\theta}_T}$ using a semi-supervised learning algorithm on $\mathcal{D}_l \cup \mathcal{D}_{ul}$
        \FOR{$ t = 1 , \cdots, T$}
            \FOR{$ b = 1 , \cdots, B$}
                \FOR{$ k = 1 , \cdots, K$}
                \STATE Generate $\widehat{\bm{x}}^{\text{pgd}}_{b, k}$ using PGD$^{10}$ in (\ref{pgd})
                ; $\bm{x}_{b, k}  \in \mathbb{R}^{d}$
                \ENDFOR
                \STATE \begin{align*}
                            \bm{\theta} \leftarrow 
                            \bm{\theta}-\eta\frac{1}{K} \nabla_{\bm{\theta}} & \mathcal{R}_{(\text{SRST-AWR})}({\bm{\theta}}; \mathcal{D}^{K}_{l} \cup \mathcal{D}^{K}_{ul}) \text{\;in\;} (\ref{our:emp-risk})
                        \end{align*}
            \ENDFOR
        \ENDFOR
    \STATE \textbf{Return} $f_{\bm{\theta}}$
	\end{algorithmic}
	\caption{Semi-Supervised Robust-Self-Training with Adaptively Weighted Regularziation Algorithm (SRST-AWR)}
    \label{alg:SRST-AWR}
\end{algorithm}

In summary,
we combine the adpaptively weighted surrogate loss and knowledge distillation with a semi-supervised teacher.
We call the algorithm Semi-Robust-Self-Training with Adaptively Weighted Regularization (SRST-AWR) which minimizes the following regularized empirical risk:
\begin{align}
    \label{our:emp-risk}
    &\mathcal{R}_{\text{(SRST-AWR)}}(\bm{\theta} ; \left\{(\bm{x}_i, y_i) \right\}_{i=1}^{n_l}, \left\{\bm{x}_j \right\}_{j=1}^{n_{ul}}, \alpha, \lambda, \beta, \gamma, \tau, \bm{\theta}_{T} ) 
    \nonumber \\
    &:= \frac{1}{n_{l}}\sum\limits_{i=1}^{n_{l}}  \ell_{\alpha}^{\text{LS}}(f_{\bm{\theta}}(\bm{x}_i), y_i) \nonumber \\
    & + \gamma \cdot \frac{1}{n_{ul}} \sum\limits_{j=1}^{n_{ul}} \operatorname{D_{KL}}(\mathbf{p}^{\tau}_{\bm{\theta}_{T}}(\cdot|\bm{x}_j) || \mathbf{p}^{\tau}_{\bm{\theta}}(\cdot|\bm{x}_j))  \nonumber \\
    & + \lambda \cdot \frac{1}{n_{ul}} \sum\limits_{j=1}^{n_{ul}}   \cdot \operatorname{D_{KL}}(\mathbf{p}_{\bm{\theta}}(\cdot|\bm{x}_j) || \mathbf{p}_{\bm{\theta}}(\cdot|\widehat{\bm{x}}^{\text{pgd}}_j)) \cdot w_{\bm{\theta}}(\bm{x}_j ; \beta\ \bm{\theta}_T),
\end{align}
where
\begin{align}
\small
    w_{\bm{\theta}}(\bm{x}_j ; \beta, \bm{\theta}_T) & = \beta \cdot \sum\limits_{c=1}^C p_{\bm{\theta}_T}(Y=c|\bm{x}_j)p_{\bm{\theta}}(c | \bm{x}_j) \nonumber \\
     + (1-  \beta ) & \cdot \left(1 - \sum\limits_{c=1}^C p_{\bm{\theta}_T}(Y=c|\bm{x}_j)p_{\bm{\theta}}(c | \widehat{\bm{x}}^{\text{pgd}}_j) \right), 
\end{align}
where
$\beta \in [0, 1]$ 
and $\tau$ is a temperature for knowledge distillation.
The proposed semi-supervised adversarial training algorithm is summarized in Algorithm \ref{alg:SRST-AWR}. 

\paragraph{Comparison of SRST-AWR to RST \cite{carmon2019unlabeled}}
The proposed SRST-AWR algorithm differs from RST \cite{carmon2019unlabeled} in three main ways.
Firstly, SRST-AWR employs a semi-supervised teacher that is trained with both labeled and unlabeled data, whereas RST utilizes a supervised teacher that is trained only on labeled data.
Second, RST uses hard targets (one-hot labels) as pseudo-labels, while SRST-AWR employs soft targets (predictive probabilities) via knowledge distillation.
The parameter $\tau$ in (\ref{our:emp-risk}) regulates the smoothness of the pseudo-labels, where smaller values result in harder targets (using one-hot labels) and larger values result in softer targets (using predictive probabilities via knowledge distillation).
Finally, SRST-AWR uses a new regularized empirical risk motivated by two upper bounds (\ref{semi-arow:ub}) and (\ref{semi-cow:ub}) for the robust risk.
Table \ref{table:compare-surrogate} provides the 
comparison of SRST-AWR, RST, UAT++ and ARMOURED.

\paragraph{Knowledge Distillation in Semi-Supervised Learning}

We have observed that the performance of a student model 
does not surpass the teacher model in the vanilla semi-supervised setting (i.e., without adversarial robust training). See Appendix \ref{knowledge-sup-semi} for empirical evidences. However, we have found that applying knowledge distillation with a semi-supervised teacher can improve the adversarial robustness and generalization performance of the student model simultaneously in SS-AT.

\paragraph{Interpretation of $w_{\bm{\theta}}(\bm{x} ; \beta, \bm{\theta}_T)$} We set $\beta$ to 1/2. Then, 
\begin{align*}
    & 2 w_{\bm{\theta}}(\bm{x} ; \beta, \bm{\theta}_T) = \sum\limits_{c=1}^C p_{\bm{\theta}_T}(Y=c|\bm{x})p_{\bm{\theta}}(c | \bm{x}) \\
    & + \left(1 - \sum\limits_{c=1}^C p_{\bm{\theta}_T}(Y=c|\bm{x}_j)p_{\bm{\theta}}(c | \widehat{\bm{x}}^{\text{pgd}}) \right) \\
    & = \left\langle p_{\bm{\theta}_T}(\cdot | \bm{x}),  p_{\bm{\theta}}(\cdot | \bm{x}) \right\rangle + (1 -  \left\langle p_{\bm{\theta}_T}(\cdot | \bm{x}),  p_{\bm{\theta}}(\cdot | \widehat{\bm{x}}^{\text{pgd}})  \right\rangle \\
    & \approx \underbrace{\operatorname{corr}(p_{\bm{\theta}_T}(\cdot | \bm{x}),  p_{\bm{\theta}}(\cdot | \bm{x}))}_{:= \text{(a)}} + (1 - \underbrace{\operatorname{corr}(p_{\bm{\theta}_T}(\cdot | \bm{x}),  p_{\bm{\theta}}(\cdot | \widehat{\bm{x}}^{\text{pgd}}))}_{:= \text{(b)}} )
\end{align*}

Our $w_{\bm{\theta}}(\bm{x} ; \beta, \bm{\theta}_T)$ in (\ref{our:emp-risk}) is designed to impose more weights
when the current prediction on a clean sample is highly correlated to the prediction of the teacher model (i.e. (a) part) and/or 
the current prediction on adversarial example is lowly correlated to the prediction of the teacher (i.e. (b) part ).

\begin{figure*}
\begin{center}
\begin{minipage}[c]{0.45\linewidth}
\includegraphics[width=\linewidth]{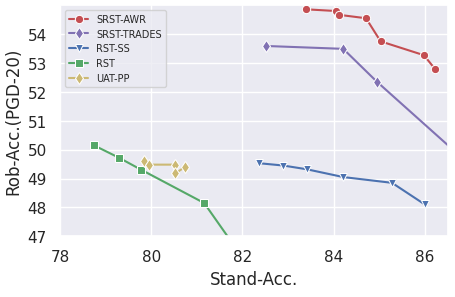}
\end{minipage}
\hfill
\begin{minipage}[c]{0.45\linewidth}
\includegraphics[width=\linewidth]{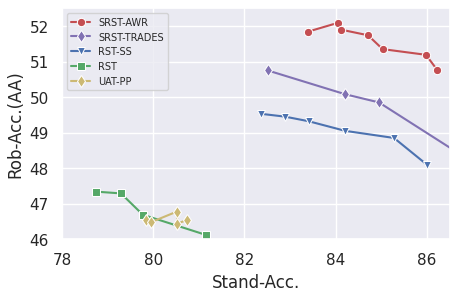}
\end{minipage}
\caption{\textbf{Comparison of SRST-AWR, SRST-TRADES, Semi-RST, RST and UAT++ for varying $\lambda$}.
The $x$-axis and $y$-axis are the standard and robust accuracies, respectively. 
The adversarial attacks for robust accuracies
are PGD$^{20}$ in the left panel and AutoAttack in right panel.
}
\label{fig:compare-tradeoff}
\end{center}
\end{figure*}

\section{Experiments}
\label{sec4}

In this section, we report the performance of our algorithm and compare it with other competitors (Section \ref{performance-evaluation}), investigate how each component of our algorithm affects the performance improvement (Section \ref{ablation-study}), and evaluate the performance of our algorithm in comparison to supervised adversarial training algorithms (Section \ref{supervised}).
The code is available at \href{https://github.com/dyoony/SRST_AWR}{https://github.com/dyoony/SRST\_AWR.}

\subsection{Experimental Setup}

\paragraph{Training Setup}
The datasets are normalized into [0, 1].
We consider the three architectures - WideResNet-28-5 (WRN-28-5) \cite{zagoruyko2016wide} for CIFAR-10  \cite{krizhevsky2009cifar} and STL-10 \cite{coates2011stl}, WideResNet-28-2 (WRN-28-2) for SVHN \cite{netzer2011svhn} and WideResNet-28-8 (WRN-28-8) for CIFAR-100 \cite{krizhevsky2009cifar}, respectively. The architecture of the teacher network is same as that of the student and the training algorithm for teacher networks is FixMatch \cite{sohn2020fixmatch}.
Details for training the teacher model are summarized in Appendix \ref{appB:teacher-models}.

We retain 4,000 samples for CIFAR-10  and CIFAR-100 and 1,000 samples for SVHN and STL-10 as labeled data
from official train data, and use the remaining data for unlabeled data.
For training prediction models,
the SGD with momentum $0.9$, weight decay  $5 \times 10^{-4}$, the initial learning rate 0.1 on CIFAR-10 and 0.05 on CIFAR-100, SVHN and STL-10 are used.
The total epochs is 200 and the learning rate is multiplied by 0.1 after each 50 and 150 epoch.
The batch size of labeled data and unlabeled data are 64 and 128, respectively.
For CIFAR-10, CIFAR-100 and STL-10, the random crop and random horizontal flip with probability 0.5 are applied for data augmentation.
Stochastic weighting average (SWA) \cite{izmailov2018averaging} is applied after 50-epochs. We select the models with having a maximum robust accuracy against PGD$^{10}$ on test set.
For training SRST-AWR,
we set $(\alpha, \gamma, \beta, \tau)$ to (0.2, 4, 0.5, 1.2) and select $\lambda$ maximizing the robust accuracy.

In maximization step, PGD$^{10}$  with random start, $p=\infty$, $\varepsilon = 8/255$ and $\nu = 2/255$ is used.
The final model is set to be the best model against PGD$^{10}$ on the validation set among those obtained until 200 epochs.
Experimental details are summarized in Appendix \ref{appB}.

\paragraph{Evaluation Setup}
For evaluating adversarial robustness, we set the maximum perturbation $\varepsilon$ to $8/255$.
We implement PGD$^{20}$ and Auto-Attack (AA) \cite{croce2020reliable}. Auto-Attack is ensemble of four attack - APGD, APGD-DLR, FAB \cite{croce2020minimally} and Square Attack \cite{andriushchenko2020square}. Among them, AGPD-DLR and Square Attack are effective to check the gradient masking \cite{ilyas2018blackbox}.


\begin{table*}
    \caption{\textbf{Comparison to RST and UAT++}. We conduct the experiment three times with different seeds and present the averages of the accuracies with the standard errors in the parenthesis.}
    \centering
    \begin{tabular}{c|ccc|ccc}
    \hline
    \multirow{2}{*}{\textbf{Method}} &
        \multicolumn{3}{c|}{CIFAR-10 (WRN-28-5)} &
        \multicolumn{3}{c}{CIFAR-100 (WRN-28-8)} \\
    \cline{2-7}    
     & \textbf{Stand}  & $\textbf{PGD}^{20}$ & \textbf{AA} & \textbf{Stand}  & $\textbf{PGD}^{20}$ & \textbf{AA} \\
     \hline
    RST                     & 79.77(0.06) & 49.31(0.09) & 46.69(0.02) & 47.36(0.06) & 	17.32(0.06) & 14.65(0.19) 
 \\
    UAT++                   & 79.84(0.23) & 49.61(0.07) & 46.53(0.11) & 46.99(0.18) 	& 22.13(0.04)  & 19.69(0.07)
 \\
   \hline
    SRST-AWR                & \textbf{84.56}(0.06) & \textbf{54.41}(0.15) & \textbf{51.72}(0.06) & \textbf{54.25}(0.04) & \textbf{30.90}(0.07) & \textbf{25.58}(0.05) \\
    \hline
    \hline
    \multirow{2}{*}{\textbf{Method}} &
        \multicolumn{3}{c|}{SVHN (WRN-28-2)} &
        \multicolumn{3}{c}{STL-10 (WRN-28-5)} \\
    \cline{2-7}    
     & \textbf{Stand}  & $\textbf{PGD}^{20}$ & \textbf{AA} & \textbf{Stand}  & $\textbf{PGD}^{20}$ & \textbf{AA} \\
     \hline
    RST                     & 86.47(0.51) & 52.62(0.44) & 42.38(0.48) & 60.81(0.32) & 35.86(0.37) & 34.00(0.22) \\
    UAT++                   & 91.07(0.08) & 52.54(0.16) & 46.81(0.13) & 58.56(0.28) & 43.00(0.21) & 40.08(0.24) \\
   \hline
    SRST-AWR               & \textbf{91.86}(0.19) & \textbf{57.02}(0.73) & \textbf{50.84}(0.83) & \textbf{89.61}(0.21) & \textbf{73.93}(0.28) & \textbf{69.99}(0.27)   \\
    \hline
  \end{tabular}
  \label{table:compare-performance}
\end{table*}

\begin{table*}
    \caption{\textbf{Comparison to ARMOURED}. The performance measures of ARMOURED are cited from original paper i.e. ARMOUREDs are not reimplemented. Maximum perturbation size $\varepsilon$s are set to be 8/255 and 4/255 on CIFAR-10 and SVHN, respectively. This settings are identical to ARMOURED \cite{zhang2021armoured}. We conduct the experiment three times with different seeds and present the averages of the accuracies with the standard errors in the parenthesis.}
    \centering
    \begin{adjustbox}{center,max width=0.95\linewidth}
    \begin{tabular}{c|ccc|ccc}
    \hline
    \multirow{2}{*}{Method} &
        \multicolumn{3}{c|}{CIFAR-10 (WRN-28-2)} &
        \multicolumn{3}{c}{SVHN (WRN-28-2)} \\
    \cline{2-7}    
     & \textbf{Stand} & $\textbf{PGD}^{7}$ & \textbf{AA} & \textbf{Stand}  & $\textbf{PGD}^{7}$ & \textbf{AA} \\
    \hline
    ARMOURED-F+AT                  & 76.76(1.60) & \textbf{55.12}(4.90) & 35.24(4.56) & 92.44(0.64) & 62.10(8.39) & 23.35(3.22)   \\
    SRST-AWR        & \textbf{81.78}(0.03) & 52.43(0.10) & \textbf{47.41}(0.11) & \textbf{95.30}(0.05) & \textbf{81.19}(0.03)  & \textbf{78.21}(0.03) \\
    \hline
    \end{tabular}
    \end{adjustbox}
    \label{table:compare-armoured}
\end{table*}

\subsection{Performance Evaluation}
\label{performance-evaluation}

\paragraph{Comparison of SRST-AWR to RST and UAT++}

We compare SRST-AWR to RST and UAT++.
Since \cite{uesato2019are} observes that UAT++ outperforms UAT-FT, we only compare UAT++ among UAT algorithms.
Figure \ref{fig:numlabels} shows the performance with varying the number of labeled data.
When the labeled data are scarce, SRST-AWR is far superior to the other competitors.
SRST-AWR maintains performance even though the size of labeled data is very small. 
In addition,
Figure \ref{fig:compare-tradeoff} shows the trade-off between the standard and robust accuracies as the regularization parameter $\lambda$ varies for SRST-AWR, RST and UAT++. 
Moreover, SRST-AWR uniformly outperforms the other competitors with large margins.
Table \ref{table:compare-performance} shows that SRST-AWR outperforms the other semi-supervised adversarial training algorithms for various benchmark data sets in terms of both standard and robust accuracies.
Experimental details of Table \ref{table:compare-performance} are provided in Appendix \ref{appB:hyper}.

\paragraph{Comparison to ARMOURED \cite{zhang2021armoured}}
We compare SRST-AWR with ARMOURED. As the official code of ARMOURED is not available, we set the architecture of SRST-AWR to be equal to that of ARMOURED and cite the performance measures of ARMOURED (standard accuracies, PGD-7, and AA) reported in the original paper. It is observed that the robust accuracy against PGD of ARMOURED is high, but not against AA, mainly due to gradient masking \cite{papernot2017practical}.
Causing gradient masking is considered as a failed defense method in terms of adversarial robustness since adversarial examples are easily generated from a model with gradient masking. Since AA contains several attack algorithms that break down models with gradient masking, it is a more reliable evaluation protocol for adversarial robustness. Table \ref{table:compare-armoured} shows that SRST-AWR outperforms ARMOURED with significant margins for CIFAR-10 and SVHN.
Experimental details of Table \ref{table:compare-armoured} are provided in Appendix \ref{appB:hyper}.

\subsection{Ablation Studies}
\label{ablation-study}
In this section, 
we investigate
how each component of the objective function (\ref{our:emp-risk}) for SRST-AWR has influence on performance - (1) effect of the semi-supervised teacher, (2) effect of the $w_{\bm{\theta}}(\bm{x} ; \beta, \bm{\theta}_T)$, 
(3) effect of the knowledge distillation (i.e. soft pseudo-labeling),
(4) sensitivity of the regularization parameter $\lambda$.
Also, we report (5) the results for fully labeled data setting to justify tightness of our bound.

Additionally, we provide the sensitivity analysis of parameter $\beta$ and $\tau$ in Appendix \ref{sensitivity-beta} and \ref{sensitivity-tau}, respectively.
Unless otherwise stated, 
the ablation studies are implemented on CIFAR-10 with 4,000 labeled data and
$(\lambda,\gamma, \tau, \beta) = ( 20, 4, 1.2, 0.5)$ in (\ref{our:emp-risk}) are used.

\subsubsection{Effect of Semi-Supervised Teacher}
\label{semi-teacher}

Table \ref{table:teacher-semi} shows that using a semi-supervised teacher improves the performance of RST \cite{carmon2019unlabeled} and UAT++ \cite{uesato2019are} on both standard and robust accuracies.
The enhancement is accomplished by assigning labels to the unlabeled data more correctly.
When using a supervised teacher, the two methods show comparable performance.
However, when using a semi-supervised teacher, RST outperforms UAT++ in terms of robustness, while UAT++ performs better than RST in terms of generalization.
However, SRST-AWR still outperforms RST and UAT++ with the semi-supervised teacher.

\begin{table}
    \small
    \caption{\textbf{Effect of Semi-supervised Teacher}. We conduct the experiments three times with different seeds and present the averages of the accuracies with the standard errors in the parenthesis.}
    \centering
    \setlength{\tabcolsep}{4pt}
    \begin{tabular}{c|ccc}
    \hline
    \multirow{2}{*}{\textbf{Method}} & \multicolumn{3}{c}{CIFAR-10 (WRN-28-5)}\\
    \cline{2-4}    
    & \textbf{Stand} & $\textbf{PGD}^{20}$ & \textbf{AA} \\
    \hline
    RST w/ sup.                              & 79.77(0.06) & 49.31(0.09) & 46.69(0.02)   \\
    RST w/ semi-sup.                         & 83.41(0.08) & 51.94(0.04) & 49.32(0.03)
   \\
   \hline
    UAT++ w/ sup.                              &  79.84(0.23) & 49.61(0.07) & 46.53(0.11) \\
    UAT++ w/ semi-sup.                         & 84.49(0.15) & 51.18(0.10) & 48.46(0.07)\\
    \hline
    RST-AWR               & 82.40(0.05) & 52.16(0.11) & 49.34(0.08) \\
    SRST-AWR              & \textbf{84.56}(0.06) & \textbf{54.41}(0.15) & \textbf{51.72}(0.06) \\
    \hline
  \end{tabular}
  \label{table:teacher-semi}
\end{table}

\subsubsection{Effect of $w_{\bm{\theta}}(\bm{x} ; \beta, \bm{\theta}_T)$ in Low-Label Regime}
\label{SRST-TRADES}

In this subsection, we compare the SRST-TRADES, which is
SRST-AWR with $w_{\bm{\theta}}(\bm{x} ; \beta, \bm{\theta}_T)=1$ for all $\bm{x}$,
and SRST-AWR for confirming the role of our proposed weight $w_{\bm{\theta}}(\bm{x} ; \beta, \bm{\theta}_T)$ in low-label regime.
The number of labeled data is 500, 2,000, 100 and 1,000 on CIFAR-10, CIFAR-100, SVHN and STL-10, respectively.
Table \ref{table:srst-awr-trades} demonstrates that SRST-AWR significantly outperforms SRST-TRADES in terms of both standard and robust accuracies.
We also compare the performance of SRST-TRADES and SRST-AWR with varying the number of labeled data on CIFAR-100 and STL-10 in Appendix \ref{compare-numlabels-cifar100-stl10}.
Details of the experimental setup are provided in Appendix \ref{appB:compare-SRST-TRADES}.


\begin{table}
    \caption{\textbf{Effect of $w_{\bm{\theta}}(\bm{x} ; \beta, \bm{\theta}_T)$} in Low-Label Regime. We conduct the experiment three times with different seeds and present the averages of the accuracies with the standard errors in the parenthesis.}
    \centering
    \setlength{\tabcolsep}{4pt}
    \begin{tabular}{c|ccc}
    \hline
    \multirow{2}{*}{\textbf{Method}} & \multicolumn{3}{c}{CIFAR-10 (WRN-28-5)}\\
    \cline{2-4}    
     & \textbf{Stand} & $\textbf{PGD}^{20}$ & \textbf{AA} \\
    \hline
    SRST-TRADES                & 79.09(0.10) & 45.51(0.15) & 43.93(0.14) \\
    SRST-AWR                   & \textbf{ 82.93}(0.06) & \textbf{52.54}(0.15) & \textbf{50.76}(0.06)  \\
    \hline
    \multirow{2}{*}{\textbf{Method}} & \multicolumn{3}{c}{CIFAR-100 (WRN-28-8)} \\
    \cline{2-4}
    & \textbf{Stand} & $\textbf{PGD}^{20}$ & \textbf{AA} \\
    \hline
    SRST-TRADES            & 29.74(0.15) & 16.51(0.21) & 14.43(0.14) \\
    SRST-AWR               & \textbf{33.65}(0.14) & \textbf{18.64}(0.17) & \textbf{16.77}(0.15) \\
    \hline
    \multirow{2}{*}{\textbf{Method}} & \multicolumn{3}{c}{SVHN (WRN-28-2)} \\
    \cline{2-4}
    & \textbf{Stand} & $\textbf{PGD}^{20}$ & \textbf{AA} \\
    \cline{2-4}
    \hline
    SRST-TRADES            & 72.85(0.11) & 47.64(0.41) & 45.08(0.31) \\
    SRST-AWR               & \textbf{80.09}(0.19) & \textbf{50.19}(0.73) & \textbf{48.69}(0.83) \\
    \hline
    \multirow{2}{*}{\textbf{Method}} & \multicolumn{3}{c}{STL-10 (WRN-28-5)} \\
    \cline{2-4}
    & \textbf{Stand} & $\textbf{PGD}^{20}$ & \textbf{AA} \\
    \cline{2-4}
    \hline
    SRST-TRADES            & 88.19(0.12) & 66.71(0.21) & 63.41(0.25) \\
    SRST-AWR               & \textbf{90.61}(0.21) & \textbf{75.85}(0.28) & \textbf{72.51}(0.27) \\
    \hline
    \hline
    \end{tabular}
    \label{table:srst-awr-trades}
\end{table}
\subsubsection{Effect of Knowledge Distillation}

To investigate the effect of knowledge distillation,
we estimate pseudo labels using a semi-supervised teacher with hard labels.
Table \ref{table:effect-kd} shows the effect of Knowledge Distillation (KD) term in our algorithm.
KD improves the robustness with retaining generalization performance.

\begin{table}
    \small
    \caption{\textbf{Effect of Knowledge Distillation}. We conduct the experiments three times with different seeds and present the averages of the accuracies with the standard errors in the parenthesis.}
    \centering
    \setlength{\tabcolsep}{4pt}
    \begin{tabular}{c|ccc}
    \hline
    \multirow{2}{*}{\textbf{Method}} & \multicolumn{3}{c}{CIFAR-10 (WRN-28-5)}\\
    \cline{2-4}    
    & \textbf{Stand} & $\textbf{PGD}^{20}$ & \textbf{AA} \\
    \hline
    SRST-AWR w/ KD  & \textbf{84.56}(0.06) & \textbf{54.41}(0.15) & \textbf{51.72}(0.06) \\
    SRST-AWR w/o KD & 84.37(0.12) & 52.77(0.09) & 50.02(0.09) \\
    \hline
    \hline
  \end{tabular}
  \label{table:effect-kd}
\end{table}


\subsubsection{Sensitivity Analysis on $\lambda$}

Figure \ref{fig:compare-tradeoff} shows the trade-off between standard accuracies and robust accuracies with respect to $\lambda$ for each algorithm.
It is observed that SRST-AWR uniformly dominates SRST-TRADES, RST-SS which is RST with the semi-supervised teacher considered in Section \ref{semi-teacher}, RST \cite{carmon2019unlabeled}, and UAT++ \cite{uesato2019are} regardless of the choice of the regularization parameter $\lambda$.

\subsubsection{Comparison to Supervised Adversarial Training Algorithms}
\label{supervised}

Table \ref{table:sup} shows that SRST-AWR with only 12\% labeled data outperforms the supervised PGD-AT \cite{madry2018towards}, TRADES \cite{zhang2019theoretically}, and MART \cite{wang2020improving} (i.e., trained with 100\% labeled data) on both standard and robust accuracies against AA, while SRST-AWR with only 8\% labeled data achieves comparable robust accuracies to the supervised TRADES, while outperforming on standard accuracy.
\begin{table}
    \footnotesize
    \caption{\textbf{Comparison to Supervised Adversarial Training Algorithms using the whole labeled data.} We conduct the experiment three times with different seeds and present the averages of the accuracies with the standard errors in the parenthesis.}
    \centering
    \scalebox{0.92}{
    \begin{tabular}{c|c|cccc}
    \hline
    \multirow{2}{*}{\textbf{Method}} & \multirow{2}{*}{\textbf{\# Labels (\%)}} & \multicolumn{3}{c}{CIFAR-10 (WRN-28-5)}\\
    \cline{3-5}    
     & & \textbf{Stand} & $\textbf{PGD}^{20}$ & \textbf{AA} \\
    \hline
    PGD-AT                              & 100 & 85.96(0.17) & 54.29(0.10) & 50.84(0.09) 
    \\
    MART                                & 100 & 80.98(0.28) & \textbf{57.56}(0.14) & 51.06(0.03) \\
    TRADES                              & 100 & 83.90(0.04) & 54.74(0.03) & 51.72(0.10) \\
    \hline
    SRST-AWR & 8   & 84.56(0.06) & 54.41(0.15) & 51.72(0.06) \\
    SRST-AWR & 12  & \textbf{86.06}(0.05) & 54.69(0.16) & \textbf{51.88}(0.08) \\    
    \hline
  \end{tabular}
  }
  \label{table:sup}
\end{table}
\subsubsection{Performance on Fully Labeled Data}

Table \ref{table:fully-labeled} shows the performance of TRADES-based methods - TRADES \cite{zhang2019theoretically} and AWP \cite{wu2020adversarial}.
AWR-based methods outperform TRADES-based methods for the given fully labeled data.
Details of the experimental setup are provided in Appendix \ref{appB:fully}.

\begin{table}
    \small
    \caption{\textbf{Performance in Fully Labeled Data}. We conduct the experiments three times with different seeds and present the averages of the accuracies with the standard errors in the parenthesis.}
    \centering
    \setlength{\tabcolsep}{4pt}
    \begin{tabular}{c|ccc}
    \hline
    \multirow{2}{*}{\textbf{Method}} & \multicolumn{3}{c}{CIFAR-10 (WRN-28-5)}\\
    \cline{2-4}    
    & \textbf{Stand} & $\textbf{PGD}^{20}$ & \textbf{AA} \\
    \hline
    TRADES  & 83.90(0.04) & 54.74(0.03) & 51.72(0.10) \\
    AWR     & \textbf{87.01}(0.10) & \textbf{55.01}(0.11) & \textbf{51.97}(0.09)  \\
    \hline
    TRADES-AWP  & 84.56(0.06) & 54.41(0.15) & 51.72(0.06) \\
    AWR-AWP & \textbf{85.81}(0.09) & \textbf{57.03}(0.14)  & \textbf{53.90}(0.12) \\
    \hline
  \end{tabular}
  \label{table:fully-labeled}
\end{table}

\section{Conclusion}

In this paper, we derived the two upper bounds of the robust risk and developed a new semi-supervised adversarial training algorithm called SRST-AWR. 
The objective function of SRST-AWR is a combination of a new surrogate version of the robust risk and a knowledge distillation term with a semi-supervised teacher.
While existing algorithms show significant performance degradation as the number of labeled data decreases, our proposed algorithm has shown relatively little performance degradation.

The experiments showed that SRST-AWR outperforms existing algorithms with significant margins in terms of both standard and robust accuracies on various benchmark datasets. Especially, we can achieve standard and robust accuracy that are comparable to supervised adversarial training algorithms when the number of labeled data is around 10\% of the whole labeled data.

\paragraph{Acknowledgement}

This work was supported by a Convergence Research Center (CRC) grant
funded by the Korean government (MSIT, No. 2022R1A5A708390811)
and by Samsung Electronics Co., Ltd.


\appendix

\numberwithin{equation}{section}
\onecolumn
{\LARGE \textbf{Appendices}}

\section{Theoretical Results}
\label{appA}

In this section, we provide the proofs of Theorems \ref{thm3.1} and \ref{thm3.2}.

\begin{lemma}(\cite{dongyoon2023improving})
\label{lemma1}
Let
\begin{equation*}
    z(\bm{x}) \in \underset{\bm{x}' \in \mathcal{B}_{p}(\bm{x}, \varepsilon)}{\operatorname{argmax}} \mathbbm{1} \{ F_{\bm{\theta}}(\bm{x}) \neq F_{\bm{\theta}}(\bm{x}')\}.
\end{equation*}
Then, 
\begin{align}
    \begin{split}
    \mathbbm{1}\left\{\exists \mathbf{X}' \in \mathcal{B}_p(\mathbf{X}, \varepsilon) : F_{\bm{\theta}}(\mathbf{X})  \neq F_{\bm{\theta}}(\mathbf{X}'), F_{\bm{\theta}}(\mathbf{X}') \neq Y\right\} \\
    \leq  \mathbbm{1}\left\{F_{\bm{\theta}}(\mathbf{X}) \neq F_{\bm{\theta}}(z(\mathbf{X})), Y \neq F_{\bm{\theta}}(z(\mathbf{X}))\right\}. \label{ineq_insung}
    \end{split}
\end{align}
The equality holds when $\mathcal{Y} = \left\{-1, 1\right\}$.
\end{lemma}

\semiarow*
\begin{proof}

Note that $\mathcal{R}_{\text{rob}}({\bm{\theta}}) = \mathcal{R}_{\text{nat}}({\bm{\theta}}) + \mathcal{R}_{\text{bdy}}({\bm{\theta}})$ where
$\mathcal{R}_{\text{nat}}(\bm{\theta}) = \mathbb{E}_{(\mathbf{X}, Y)}\mathbbm{1}\left\{ F_{\bm{\theta}}(\mathbf{X}) \neq Y \right\}$ and 
$\mathcal{R}_{\text{bdy}}(\bm{\theta}) = \mathbb{E}_{(\mathbf{X}, Y)}\mathbbm{1}\left\{\exists \mathbf{X}' \in \mathcal{B}_p(\mathbf{X}, \varepsilon) :F_{\bm{\theta}}(\mathbf{X})\neq F_{\bm{\theta}}(\mathbf{X}')   ,F_{\bm{\theta}}(\mathbf{X}) = Y   \right\}$.

Since
\begin{align*} 
\mathcal{R}_{\text{bdy}}({\bm{\theta}}) &= \mathbb{E}_{(\mathbf{X}, Y)}\mathbbm{1}\left\{\exists \mathbf{X}' \in \mathcal{B}_p(\mathbf{X}, \varepsilon) : F_{\bm{\theta}}(\mathbf{X}) \neq F_{\bm{\theta}}(\mathbf{X}'), F_{\bm{\theta}}(\mathbf{X})=Y \right\} \\ 
& \leq \mathbb{E}_{(\mathbf{X}, Y)}\mathbbm{1}\left\{F_{\bm{\theta}}(\mathbf{X}) \neq F_{\bm{\theta}}(z(\mathbf{X})), Y \neq F_{\bm{\theta}}(z(\mathbf{X}))\right\} (\because \text{\;Lemma\;} \ref{lemma1})\\
& = \mathbb{E}_{\mathbf{X}} \mathbbm{1} \left\{ F_{\bm{\theta}}(\mathbf{X}) \neq F_{\bm{\theta}}(z(\mathbf{X})) \right\} \mathbb{E}_{Y|\mathbf{X}} \mathbbm{1}\left\{ Y \neq F_{\bm{\theta}}(z(\mathbf{X})) \right\} \\ 
& = \mathbb{E}_{\mathbf{X}} \left\{ \mathbbm{1} \left\{ F_{\bm{\theta}}(\mathbf{X}) \neq F_{\bm{\theta}}(z(\mathbf{X})) \right\} \cdot p(Y \neq F_{\bm{\theta}}(z(\mathbf{X}))|\mathbf{X}) \right\},
\end{align*}
the inequality (\ref{semi-arow:ub}) holds.
\end{proof}

\semicow*
\begin{proof}
It suffices to show that $\mathcal{R}_{\text{bdy}}({\bm{\theta}}) \leq \mathbb{E}_{\mathbf{X}}\left\{{\mathbbm{1} \{ F_{\bm{\theta}}(\mathbf{X}) \neq F_{\bm{\theta}}(z(\mathbf{X}'))\}} \cdot  p(\mathbf{Y}=F_{\bm{\theta}}(\mathbf{X})|\mathbf{X})\right\}$.

Since
\begin{align*}
\mathcal{R}_{\text{bdy}}(\bm{\theta}) &= \mathbb{E}_{(\mathbf{X}, \mathbf{Y})}\mathbbm{1}\left\{\exists \mathbf{X}' \in \mathcal{B}_p(\mathbf{X}, \varepsilon) : F_{\bm{\theta}}(\mathbf{X}) \neq F_{\bm{\theta}}(\mathbf{X}'), F_{\bm{\theta}}(\mathbf{X})=\mathbf{Y} \right\} \\
& = \mathbb{E}_{(\mathbf{X}, \mathbf{Y})} \mathbbm{1} \left\{\exists \mathbf{X}'\in \mathcal{B}_p(\mathbf{X}, \varepsilon) : F_{\bm{\theta}}(\mathbf{X}') \neq F_{\bm{\theta}}(\mathbf{X}) \right\} \mathbbm{1}\left\{ \mathbf{Y} = F_{\bm{\theta}}(\mathbf{X}) \right\} \\ 
& =  \mathbb{E}_{(\mathbf{X}, \mathbf{Y})} {\mathbbm{1}\{ F_{\bm{\theta}}(\mathbf{X}) \neq F_{\bm{\theta}}(z(\mathbf{X})) \} } \mathbbm{1} \left\{ \mathbf{Y} = F_{\bm{\theta}}(\mathbf{X})\right\} \\
& = \mathbb{E}_{\mathbf{X}} \mathbbm{1}\{F_{\bm{\theta}}(\mathbf{X}) \neq F_{\bm{\theta}}(z(\mathbf{X})) \} \cdot \mathbb{E}_{\mathbf{Y|X}} \mathbbm{1} \left\{ \mathbf{Y} = F_{\bm{\theta}}(\mathbf{X})\right\} \\
& = \mathbb{E}_{\mathbf{X}} \left\{ \mathbbm{1}(F_{\bm{\theta}}(\mathbf{X}) \neq F_{\bm{\theta}}(\mathbf{X}') ) \cdot p(\mathbf{Y}=F_{\bm{\theta}}(\mathbf{X})| \mathbf{X}) \right\},
\end{align*}
the inequality (\ref{semi-cow:ub}) holds.
\end{proof}

\section{Experimental Setup}
\label{appB}

\subsection{Hyperparameter setting}
\label{appB:hyper}
\begin{table}[H]
    \centering
    \caption{\textbf{Selected hyperparameters.} Hyperparameters used in the numerical studies in Table \ref{table:compare-performance} and \ref{table:compare-armoured}.}
    \begin{tabular}{c|c|c|ccccc}
    \hline
     Dataset& Model & $\textbf{Method}$ & $ \lambda$ & $\gamma$  & $\beta$ & $\tau$ & Teacher  \\
    \hline
    \hline
    \multirow{3}{*}{\text{CIFAR-10}} & \multirow{3}{*}{\text{WRN-28-5,2}}
      &  RST & 5 & - & - & -   &  Supervised \\
    & & UAT++      & 5 & - & - & - &  Supervised \\
    & & SRST-AWR  & 20 & 4 & 0.5 & 1.2 &  FixMatch \\
    \hline
    \multirow{3}{*}{\text{CIFAR-100}} & \multirow{3}{*}{\text{WRN-28-8}}   
      &  RST & 5 & - & - & -   &  Supervised \\
    & & UAT++      & 5 & - & - & - &  Supervised \\
    & & SRST-AWR & 20 & 4 & 0.5 & 1.0 &  FixMatch \\
    \hline
    \multirow{3}{*}{\text{SVHN}} & \multirow{3}{*}{\text{WRN-28-2}} 
      &  RST & 5 & - & - & -   &  Supervised \\
    & & UAT++      & 5 & - & - & - &  Supervised \\
    & & SRST-AWR & 15 & 4 & 0.5 & 1.0 &  FixMatch \\
    \hline
    \multirow{3}{*}{\text{STL-10}} & \multirow{3}{*}{\text{WRN-28-5}} 
      &  RST & 5 & - & - & -   &  Supervised \\
    & & UAT++      & 5 & - & - & - &  Supervised \\
    & & SRST-AWR & 8 & 4 & 0.5 & 1.0 &  FixMatch \\
    \hline
    \end{tabular}
    \label{table_hyper}
\end{table}

Table \ref{table_hyper} presents the hyperparameters used in Table \ref{table:compare-performance} and \ref{table:compare-armoured}.
Most of the hyperparameters are set to be the ones used in the previous studies.

\subsection{Loss for Generating Adversarial Examples}

As described in Section \ref{adversarial-attack},
KL-divergence and cross-entropy loss can be used to generate the adversarial examples.
When using KL divergence, adversarial examples are generated based on the current predictions, whereas when using cross-entropy, the target label is required.
In our experimental setting, 
we use the cross-entropy loss with target labels which are predicted by the teacher models except for CIFAR-10 since KL-divergence are not stable.

\subsection{The teacher models}
\label{appB:teacher-models}

\subsubsection{Hyperparameter setting}
\begin{table}[H]
    \small
    \centering
    \caption{\textbf{Selected hyperparameters.} Hyperparameters used in the numerical studies in Section \ref{sec4}.}
    \label{table:hyper_teacher}
    \begin{tabular}{c|c|c|cccccc}
    \hline
     Dataset & Model & $\textbf{Method}$ & $ \lambda$ &  Weight Decay & $\tau$ & \text{num\_labels} & batch size (labeled, unlabeled)  \\
    \hline
    \hline
    \text{CIFAR-10} & \text{WRN-28-5} & FixMatch & 1 & $5e^{-4}$ & 0.95 & 4,000 & (64, 128)    \\
    \hline
    \text{CIFAR-100} & \text{WRN-28-8} & FixMatch & 1 & $5e^{-4}$ & 0.95 & 4,000 & (64, 128)  \\
    \hline
    \text{SVHN} & \text{WRN-28-2} & FixMatch & 1 & $5e^{-4}$ & 0.95 & 1,000 & (64, 128) \\
    \hline
    \text{STL-10} & \text{WRN-28-5} & FixMatch & 1 & $5e^{-4}$ & 0.95 & 1,000 & (64, 128) \\
    \hline
    \end{tabular}
\end{table}

\begin{table}[H]
    \caption{\textbf{Performance of Teachers}.}
    \label{table:teacher-performance}
    \centering
    \begin{tabular}{c|c|c|c|c}
    \hline
    & CIFAR-10 ($n_l=4,000$)& CIFAR-100 ($n_l=4,000$) & SVHN ($n_l=1,000$) & STL-10 ($n_l=1,000$)\\
    \hline
    Supervised  & 81.82 & 45.96 & 83.50 & 56.60\\
    FixMatch    & 95.87 & 64.82 & 97.11 & 92.46 \\
    \hline
    \end{tabular}
\end{table}
We train the model using FixMatch. 
Table \ref{table:hyper_teacher} and \ref{table:teacher-performance} shows the selected hyperparmeters for training teacher models and performance of them.

\subsection{Comparison SRST-AWR to SRST-TRADES}
\label{appB:compare-SRST-TRADES}
\begin{table}[H]
    \centering
    \caption{\textbf{Selected hyperparameters.} Hyperparameters used in the numerical studies in Table \ref{table:srst-awr-trades}.}
    \begin{tabular}{c|c|c|ccccc}
    \hline
    Dataset & Model & $\textbf{Method}$ & $ \lambda$ & $\gamma$  &  $\beta$ & $\tau$ & Teacher  \\
    \hline
    \hline
    \multirow{2}{*}{\text{CIFAR-10}}  & \multirow{2}{*}{\text{WRN-28-5}}
    & SRST-TRADES       & 12 & 4 & - & 1.2 &  FixMatch \\
    & & SRST-AWR          & 20 & 4 & 0.5 & 1.2 &  FixMatch \\
    \hline
    \multirow{2}{*}{\text{CIFAR-100}}  & \multirow{2}{*}{\text{WRN-28-8}}
    & SRST-TRADES       & 20 & 4 & - & 1.0 &  FixMatch \\
    & & SRST-AWR          & 20 & 4 & 0.5 & 1.0 &  FixMatch \\
    \hline
    \multirow{2}{*}{\text{STL-10}}  & \multirow{2}{*}{\text{WRN-28-5}}
    & SRST-TRADES       & 8 & 4 & - & 1.0 &  FixMatch \\
    & & SRST-AWR          & 8 & 4 & 0.5 & 1.0 &  FixMatch \\
    \hline
    \end{tabular}
    \label{table:hyper-awr-trades}
\end{table}

Table \ref{table:hyper-awr-trades} presents the hyperparameters used in Table \ref{table:srst-awr-trades}.

\subsection{Comparison SRST-AWR to Other Competitors in Supervised Setting}
\label{appB:compare-sup}
\begin{table}[H]
    \centering
    \caption{\textbf{Selected hyperparameters.} Hyperparameters used in the numerical studies in Table \ref{table:sup}.}
    \label{hyper_sup}
    \begin{tabular}{c|c|cccccc}
    \hline
    $\textbf{Method}$ & $ \lambda$ & $\gamma$ & $\beta$  & $\tau$ & $\#$ of labeled data  \\
    \hline
    \hline
    PGD-AT                 & - & - & - & - & 50,000(100\%)  \\
    TRADES                 & 6 & - & - & - & 50,000(100\%)  \\
    MART                   & 4 & - & - & - & 50,000(100\%)  \\
    SRST-AWR               & 20 & 4 & 0.5 &  1.2 & 4,000(8\%) \\
    SRST-AWR               & 20 & 4 & 0.5 &  1.2 & 6,000(12\%) \\
    \hline
    \hline
    \end{tabular}
\end{table}

Table \ref{hyper_sup} presents the hyperparameters used in Table \ref{table:compare-performance}.
Most of the hyperparameters are set to be the ones used in the
previous studies \cite{madry2018towards, zhang2019theoretically, wang2020improving}.

\subsection{Performance on Fully Labeled Data} 
\label{appB:fully}

\begin{table}[H]
    \centering
    \caption{\textbf{Selected hyperparameters.} Hyperparameters used in the numerical studies in Table \ref{table:fully-labeled}.}
    \label{hyper_fully}
    \begin{tabular}{c|cc}
    \hline
    $\textbf{Method}$ & $ \lambda$ & $\#$ of labeled data  \\
    \hline
    \hline
    TRADES                 & 6  & 50,000(100\%)  \\
    AWR                    & 9  & 50,000(100\%)  \\
    \hline
    AWP-TRADES             & 6 & 50,000(100\%)  \\
    AWP-AWR                & 9 & 50,000(100\%)  \\
    \hline
    \hline
    \end{tabular}
\end{table}

Table \ref{hyper_fully} presents the hyperparameters used in Table \ref{table:fully-labeled}. We select models with maximized robust accuracies against PGD$^{10}$.

\section{Additional Experiments}
\subsection{Effect of knowledge distillation}
\label{knowledge-sup-semi}
We conduct a comparison of knowledge distillation in both supervised and semi-supervised settings using the same architecture for the teacher and student models. 
For the supervised setting, we set $\alpha$ and $\tau$ to 0.9 and 20, respectively, as suggested in \cite{hinton2015distilling}. 
For the semi-supervised setting, we perform a grid search and found that setting $\alpha$ to 0.9 and $\tau$ to 1.1 yields the optimal results.
The results are presented in Table \ref{table:knowledge-sup-semi}, which show that while the student model can outperform the teacher model in the supervised setting, it cannot achieve comparable performance to the teacher model in the semi-supervised setting.

\begin{table}[H]
    \centering
    \caption{\textbf{The Effect of Knowledge Distillation}.}
    \begin{tabular}{c|cccc}
    \hline
    $\textbf{Setting}$ & $ \text{Teacher Acc.} $ & \text{Student Acc.} & Diff. &  \# of labeled data \\
    \hline
    \hline
    Supervised             & 96.04(0.11) & 96.14(0.09) & +0.1 & 50,000(100\%) \\
    Semi-supervised        & 95.87(0.05) & 95.37(0.04) & -0.5 & 4,000(8\%) \\
    \hline
    \end{tabular}
    \label{table:knowledge-sup-semi}
\end{table}

\subsection{The comparison SRST-AWR and SRST-TRADES with varying the number of labeled data}
\label{compare-numlabels-cifar100-stl10}

We compare the performance of SRST-AWR and SRST-TRADES with varying numbers of labeled data for CIFAR-100 and STL-10 to assess the effect of $w_{\bm{\theta}}(\bm{x} ; \beta, \bm{\theta}_T)$.
Figure \ref{fig:num-labels-cifar100} and \ref{fig:num-labels-stl10} present the results, showing that SRST-AWR consistently outperforms SRST-TRADES for both datasets across all labeled data sizes. For CIFAR-100, we use 4,000, 6,000, 10,000, and 15,000 labeled data , while for STL-10, we use 100, 250, 500, 1,000 and 2,000 labeled data. For both datasets, we observe that both standard and robust accuracies against AA 
improve
as the number of labeled data increases.
The margins between SRST-AWR and SRST-TRADES are relatively high, especially when the number of labeled data is 4,000 on CIFAR-100 and 1,000 on STL-10. Overall, our results demonstrate that SRST-AWR can outperform SRST-TRADES, even with limited amounts of labeled data, indicating the efficacy of $w_{\bm{\theta}}(\bm{x} ; \beta, \bm{\theta}_T)$ in improving adversarial robustness.

\begin{figure}[H]
\begin{center}
\begin{minipage}[c]{0.45\linewidth}
\includegraphics[width=\linewidth]{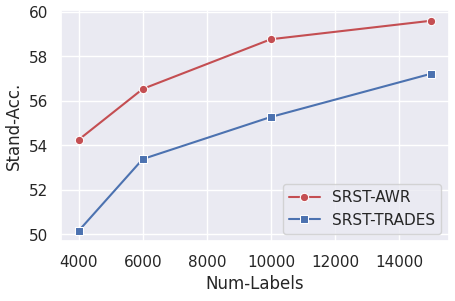}
\caption*{(a) y-axis : standard accuracies}
\end{minipage}
\hfill
\begin{minipage}[c]{0.45\linewidth}
\includegraphics[width=\linewidth]{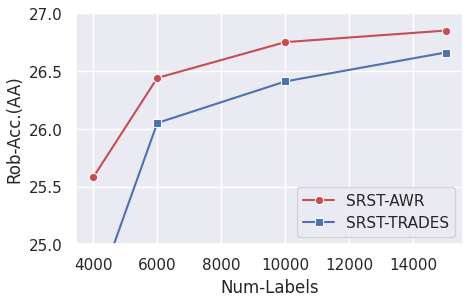}
\caption*{(b) y-axis : robust accuracies}
\end{minipage}
\caption{\textbf{Comparison SRST-AWR to SRST-TRADES with varying the number of labeled data on CIFAR100}.
The $x$-axis and $y$-axis are the number of labeled data and performance measure - standard accuracies and robust accuracies against AA , respectively.}
\label{fig:num-labels-cifar100}
\end{center}
\end{figure}
\begin{figure}[H]
\begin{center}
\begin{minipage}[c]{0.45\linewidth}
\includegraphics[width=\linewidth]{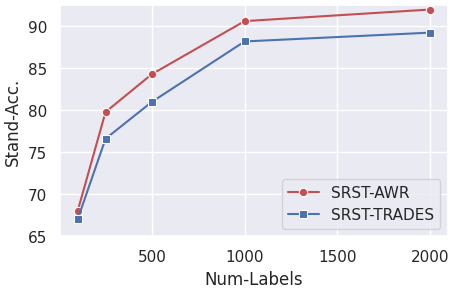}
\caption*{(a) y-axis : standard accuracy}
\end{minipage}
\hfill
\begin{minipage}[c]{0.45\linewidth}
\includegraphics[width=\linewidth]{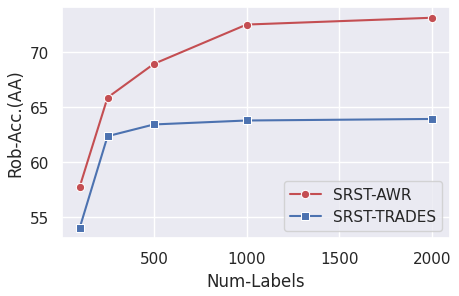}
\caption*{(b) y-axis : robust accuracy}
\end{minipage}
\caption{\textbf{Comparison SRST-AWR to SRST-TRADES with varying the number of labeled data on STL-10}.
The $x$-axis and $y$-axis are the number of labeled data and performance measure - standard accuracies and robust accuracies against AA , respectively.}
\label{fig:num-labels-stl10}
\end{center}
\end{figure}

\subsection{Sensitivity Analysis on $\beta$}
\label{sensitivity-beta}
In this subsection, we perform a sensitivity analysis on $\beta$. Figure \ref{sensitivity:beta} illustrates that the highest level of robustness can be attained at $\beta=0.5$ on CIFAR-10. Additionally, standard accuracy remains relatively high when $\beta$ is in the range of [0, 0.5].

\begin{figure}[H]
    \centering
    \includegraphics[width=8cm, height=5cm]{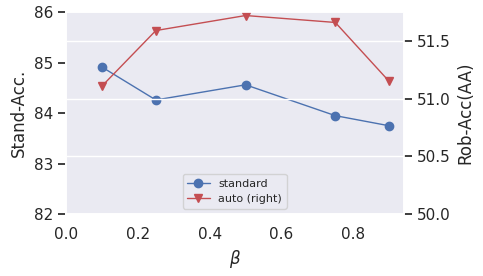}
    \caption{\textbf{Sensitivity Analysis for $\beta$}. 
    We vary $\beta$ from 0 to 1 in SRST-AWR.
    The $x$-axis is $\beta$
    and $y$-axes are  standard accuracy and robust accuracy against AA, respectively.}
\label{sensitivity:beta}
\end{figure}
\subsection{Sensitivity Analysis of Temperature $\tau$}
\label{sensitivity-tau}
We perform a sensitivity analysis on the temperature parameter $\tau$ for knowledge distillation on CIFAR-10. Figure \ref{sensitivity:tau} demonstrates that the selection of $\tau$ affects both standard and robust accuracies. Specifically, increasing $\tau$ to 1.4 enhances robustness but deteriorates generalization. On the other hand, if $\tau$ exceeds 1.2, both robustness and generalization decline. Therefore, the results obtained with $\tau = 1.2$ are more favorable compared to other choices for enhancing robustness on CIFAR-10.

\begin{figure}[H]
    \centering
    \includegraphics[width=8cm, height=5cm]{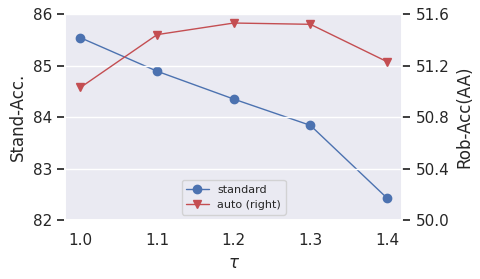}
    \caption{\textbf{Sensitivity Analysis for $\tau$}. 
    We vary $\tau$ from 1 to 1.4 in SRST-AWR.
    The $x$-axis is $\tau$
    and $y$-axes are standard accuracy and robust accuracy against AA, respectively.}
    \label{sensitivity:tau}
\end{figure}

\end{document}